\documentclass{article}

\usepackage[final]{corl_2020} 

\usepackage{graphicx}
\usepackage{subcaption}

\usepackage{latexsym, amscd, amsfonts, eucal, mathrsfs, amsmath, amssymb, amsthm, xypic,xr, stmaryrd, color, enumerate}
\usepackage{multicol}
\usepackage{amsmath}
\usepackage{amsmath,amssymb}
\usepackage{bm}
\usepackage{bbm}

\usepackage{algorithm}
\usepackage[noend]{algpseudocode}

\usepackage{xspace}
\usepackage[dvipsnames]{xcolor}
\usepackage[textsize=footnotesize]{todonotes}

\newtheorem{theorem}{Theorem}[section]

\newtheorem{assumption}[theorem]{Assumption}

\theoremstyle{definition}
\newtheorem{definition}[theorem]{Definition}

\title{Concurrent Training Improves the Performance of Behavioral Cloning from Observation}

\author{
  Zachary W. Robertson\\\
  University of Chicago
  \texttt{zrobertson@uchicago.edu} \AND
  Matthew R. Walter\\
  Toyota Technological Institute at Chicago
  \texttt{mwalter@ttic.edu} }

\begin{document}
\maketitle


\begin{abstract}
    Learning from demonstration is widely used as an efficient way for robots to acquire new skills. However, it typically requires that demonstrations provide full access to the state and action sequences. In contrast, learning from observation offers a way to utilize unlabeled demonstrations (e.g., video) to perform imitation learning. One approach to this is behavioral cloning from observation (BCO). The original implementation of BCO proceeds by first learning an inverse dynamics model and then using that model to estimate action labels, thereby reducing the problem to behavioral cloning. However, existing approaches to BCO require a large number of initial interactions in the first step. Here, we provide a novel theoretical analysis of BCO, introduce a modification BCO*, and show that in the semi-supervised setting, BCO* can concurrently improve both its estimate for the inverse dynamics model and the expert policy. This result allows us to eliminate the dependence on initial interactions and dramatically improve the sample complexity of BCO. We evaluate the effectiveness of our algorithm through experiments on various benchmark domains. The results demonstrate that concurrent training not only improves over the performance of BCO but also results in performance that is competitive with state-of-the-art imitation learning methods such as GAIL and Value-Dice.
\end{abstract}


\section{Introduction}

Reinforcement learning (RL) allows a robot to learn in an environment directly from experience, without knowledge of the environment dynamics or rewards. However, theoretical results indicate that the sample complexity of RL scales poorly unless strong assumptions on the environment are made ~\cite{RL_Hard}. Poor sample complexity is particularly undesirable for robotics and other domains where interaction with the environment may be slow, expensive, difficult, or dangerous.

In contrast, imitation learning provides a more sample efficient alternative to learning from experience, whereby the robot learns to emulate optimal behavior conveyed through expert demonstrations~\cite{IL_Easy}. One approach to learning from demonstrations is behavioral cloning (BC) which is a supervised approach in which the robot learns to predict what action the expert would take at each state~\cite{bain1995framework,daftry2016learning}.

Progress in imitation learning has enabled robots to perform a diverse array of complex tasks~\cite{chen2017end,giusti2015machine,abbeel2010autonomous}. However, most existing methods assume that the robot observes the expert's actions as part of the demonstrations. This assumption imposes a protocol on the process of collecting demonstration data and precludes the use of a large amount of crowd-sourced demonstration data that provides only observations, such as YouTube videos. This has become particularly limiting in recent years due to the advantages and, in turn, the prevalence of neural network-based representation learning, which has proven successful for complex supervised and reinforcement learning tasks \cite{deng2009imagenet,mnih2015human,gu2017deep}. For example, \citet{smith2019avid} uses CycleGAN to transform videos of human demonstrations into robot demonstrations that are easier to imitate. 

To take advantage of this information, a more general problem of imitation learning from observation (ILfO) must be considered. \citet{1805.01954} introduced a generalization known as behavioral cloning from observation (BCO) to address this problem. By first learning how to model the dynamics from observation, the robot can infer what actions to take by using the inverse dynamics model to assign control information to the demonstration data set. 

The original implementation of BCO requires a large number of prior interactions in the environment to learn the inverse dynamics model. This inefficiency happens because a random policy is used to collect the data. In the first part of this paper, we take a closer look at BCO and show that when interpreted as a semi-supervised learning problem, we can provide guarantees on the ability of BCO to imitate the demonstrations. This suggests a change to the original implementation that we refer to as BCO*. BCO* eliminates the random policy step and replaces it with concurrent training of the inverse dynamics model and policy. In the second part, we provide an empirical examination in both discrete and continuous domains to verify the applicability of our proposed alteration to BCO. 

\section{Related Work}

Behavioral cloning from observation is a model-based imitation learning method centered around combining an inverse dynamics model with behavioral cloning to learn an imitation policy. \citet{1805.01954} shows that this approach can match the performance of behavioral cloning using only unlabeled samples. The algorithm works in two stages. First, an inverse dynamics model is learned by allowing the robot to execute a random policy in the environment~\cite{1805.01954}. Next, the inverse dynamics model is used to label the demonstrations. Similar to this \citet{1805.07914} proposed a model-based ILfO method that allows robots to learn a latent policy as a form of unsupervised learning before interacting with the environment which is then mapped to an explicit policy. 

Another approach to imitation learning focuses on getting the robot to directly reproduce the state-distribution of the expert. One of the first distribution difference minimization algorithms was generative adversarial imitation learning (GAIL)~\cite{1606.03476}. In analogy with GANs, GAIL focuses on learning a discriminator that is capable of distinguishing expert demonstrations from robot trajectories while simultaneously training the robot to avoid detection by the discriminator~\cite{1406.2661}. Value-DICE and discriminator actor-critic (DAC) built on this work and provide sample-efficient alternatives to minimize the GAIL objective~\cite{kostrikov2019imitation,kostrikov2018discriminator}.

More recently, GAIfO and internal-disagreement minimization have been introduced as an extension of GAIL to the ILfO setting~\cite{Torabi_GAIL_State_Only, yang2019imitation}. \citet{yang2019imitation} also compare their internal-disagreement algorithm against a version of BCO with concurrent training. However, because they start with a random policy the robot fails to perform well. In this work, we show that giving the robot a better initialization leads to better performance. From this result, we propose a method that eliminates the need to execute a random policy and reduces the sample complexity of BCO by a significant amount.

Because the inverse dynamics model is used to assign pseudo-labels to the unlabeled demonstrations dataset, BCO can be seen as a semi-supervised imitation learning algorithm. The advantage of this perspective is that there is a rich literature on improving pseudo-labeling. Recent advances in semi-supervised learning have been able to achieve state-of-the-art results using entropy-based methods and consistency training, which works by forcing classifications to be invariant to noise~\cite{berthelot2019mixmatch,xie2019unsupervised}. 

It is also worth mentioning that our basic approach relies on bounding the state distribution difference between the expert and the robot. This is similar in spirit to the noise injection approach introduced by \citet{DART} that is used to alleviate the off-policy problem in imitation learning. In this approach, a noise parameter is tuned so that the demonstrator state distribution maximally overlaps with that of the robot. In this work, we merely assume the existence of a ``noisy'' expert to mitigate the off-policy problem and do not optimize this parameter. However, we are still see significant improvement in the resulting policies. 

\section{Preliminaries}

\subsection{Notation}

We briefly review the setup for behavioral cloning from observation. Following \citet{1805.01954}, we work within a Markov decision process (MDP) described as $M = \lbrace S, A, T, r, \gamma \rbrace$, where $S$ is the robot's state space, $A$ is the robot's action space, $T = P(s_{t+1} \vert s_t , a)$ is the transition function that specifies the likelihood of transitioning to state $s_{t+1}$ when executing action $a$ in state $s_t$, $r: S \times A \to \mathbb{R}$ is the reward function, and $\gamma$ is the discount factor. According to this framework, we seek to learn a policy function $\pi: S \to A$ such that the discounted sum of rewards $\sum_{t = 1}^{\infty} \gamma^t r$ is maximized. 

In the case of BCO, we're also interested in the inverse dynamics model, $M = P(a \vert s_t, s_{t+1})$ that tries to assign the most likely action given a state transition. In imitation learning, we usually assume that we have no access to the reward function. Instead, we have access to a set of demonstrations
\begin{equation}
    D_\text{demo} = \{ \zeta_1, \zeta_2, \ldots \},
\end{equation}
where $\zeta_i$ is a sequence of state-action pairs $\{ (s_0,a_0), \ldots, (s_N,a_N) \}$. However, imitation learning from observation assumes no access to action labels, and thus$\zeta_i = \lbrace s_0, \ldots, s_N \rbrace$.  

\subsection{The Off-Policy Problem}

The goal of imitation learning is to emulate the expert sufficiently so that the robot attains similar or better rewards compared to the demonstrations. In general, this problem can be quite difficult when there are certain actions the robot could take, such as falling off a cliff, that eliminate the chance of any future reward.

Let $J(\pi^*)$ be the cost incurred by the expert policy $\pi^*$, then we can estimate the cost incurred by an empirical policy $\hat \pi$ in terms of it's $0$-$1$ error $\epsilon$. We assume the policy is deterministic and write $\pi_s$ to denote the distribution of actions over the current state $s$. We let $T$ refer to the task horizon. We also define $C(s,a)$ as $C: S \times A \to [0,1]$ to be the immediate cost of taking an action $a$ in the state $s$. In the imitation learning setting, we are generally interested in the quantity $l(s,a) = \mathbbm{1} \{ a \not = \pi^*_s \}$, the occurrence of events where an action is not aligned with the expert's policy. This is a $0$-$1$ loss. This quantity is usually quite difficult to minimize directly so instead, we'll minimize the proxy $C$.

We are interested in the state distribution $d_{\pi}^t$ at time $t$ assuming we follow a policy $\pi$ from the initial time step. We denote the expected $T$-step cost of executing $\pi$ with $J(\pi) = T \mathbb{E}_{s \sim d_{\pi}}[C_{\pi}(s)]$. Given this, we can look at the expected $t$-cost of taking an action under a policy as the Q-function $Q_t^{\pi}(a,s)$. Here we state Theorem~2.2 of \citet{behavior_cloning},

\begin{theorem}
\label{Ross Theorem}
Let $\pi$ be such that $\mathbb{E}_{s \sim d_{\pi}}[l(s,\hat \pi)] = \epsilon$ and $Q_{T-t+1}^{\pi^*}(s,a) - Q_{T-t+1}^{\pi^*}(s,\pi^*) \le \frac{1}{2} u$ for all times and actions, where $d_{\pi}^t(s) > 0$, then $J(\pi) \le J(\pi^*) + \frac{1}{2} u T \epsilon$
\end{theorem}

For convenience, the proof is also provided in the supplementary materials. When there are certain actions that have a permanent impact on the robot's ability to emulate the expert $u \sim T$ and the robot could experience quadratic divergence from the expert policy's over time. On the other hand, if the demonstrations are noisy, but have comparable cost to the perfect expert then it follows that $u \ll T$. Thus, whenever we can inject noise into the expert policy with minimal change to quantities such as total return we infer that $u$ is small.

\section{Theoretical Analysis}

\subsection{Assumptions}

In order to proceed with the analysis, it will be beneficial to explicitly discuss what assumptions are needed. Broadly, we assume that the expert policy and inverse-dynamics models are learnable with reliable sample complexity, that the demonstrator is an $\epsilon$-expert, we have an arbitrarily large number of unlabeled demonstrations, and that mistakes are recoverable. 

\subsubsection{Recoverability}

Say we collect a demonstration dataset $D_\text{demo}$ from our expert. The state-action pairs will follow some distribution $d_{\pi}(s,a)$ for any given $\pi$. We assume the experts are \textit{not} perfect.

\begin{definition}

A $\epsilon$-expert policy is an expert policy $\pi^*$ that each action it outputs is replaced to a different random action with probability $\epsilon$.

\end{definition}

If $\epsilon$ is small, then Theorem~\ref{Ross Theorem} would suggest that the performance of the $\epsilon$-expert policy and that of the expert to be close. We can also estimate the relative probability that a state transition $(s,s')$ comes from $\pi^*$ or $\pi$. Since this quantity is bounded between zero and one, relative probability satisfies the conditions of Theorem~\ref{Ross Theorem} as well.

\begin{assumption}
We assume that the demonstrators consist of $\epsilon$-experts.
\end{assumption}

This implies that mistakes can be recovered from. This implies that $u \ll T$ in Theorem~\ref{Ross Theorem}. If the action space is bounded, but continuous, then we further assume that,

\begin{assumption}
Let $\pi$ be such that $\mathbb{E}_{s \sim d_{\pi}}[l(s,\hat \pi)] = \epsilon$. Assume that,
\begin{align}
\mathbb{E}[Q_{T-t+1}^{\pi^*}(s,a) - Q_{T-t+1}^{\pi^*}(s,\pi^*)] \le \frac{1}{2} u \epsilon
\end{align}
for all times and actions where $d_{\pi}^t(s) > 0$.
\end{assumption}

Applying this assumption to Theorem~\ref{Ross Theorem} (Eqn.~\ref{eqn:cost-empirical}) yields a generalized version of the theorem.

\subsubsection{Learnability}

The generalization capability of the robot's policy is directly tied to the model class used to learn the inverse and policy functions. We prefer to use a large function class, such as those learned by neural networks, but the generalization bounds generated using standard VC-dimension arguments are poor~\cite{VC_Bound}. However, recent work has shown that as the width of a neural network's hidden layers grows to infinity, the function approximations converge to a Gaussian process \cite{NTK_Gaussian_Process}. Thus, we focus on kernel methods. Since the specific choice of the kernel is a detail, we'll assume that generalization under the 0-1 loss $L_D^{0-1}(\theta)$ for model parameters $\theta$ can be bounded in terms of an appropriate complexity measure.

\begin{assumption}
\label{NTK Generalization}
Let $\bm{y} = (y_1, \ldots, y_n)^T$. Assume our model class is such that with probability at least $1 - \delta$ the output of our algorithm satisfies,
\begin{align}
\mathbb{E} \left[L_{D}^{0-1}(\bm{\theta}) \right] \le O \left( \sqrt{\cfrac{K(\bm{y})}{n}} + \sqrt{\frac{\log(1/\delta)}{n}}\right)
\end{align}
where $K$ is a kernel-based complexity measure of the data and $n$ is the training set size.

\end{assumption}

For kernel methods, we have $K = \bm{y}^\top (\Theta^{(L)})^{-1} \bm{y}$ where $K$ is a measure of the complexity for the learning problem. Essentially, if our target function $y = f^*(\bm{x})$ has a bounded norm in the induced reproducing kernel Hilbert space (RKHS), then we can expect our algorithm to generalize well. 

\begin{assumption}
\label{inverse_learn}
We assume that the true policy and inverse dynamics model being learned have bounded RKHS norm.
\end{assumption}

We denote these quantities by $\lVert \pi^* \rVert_{K}$, $\lVert M^*\rVert_{K}$, or $K$ when it is ambiguous. Finally, 

\begin{assumption}
We have access to many more unlabeled demonstrations than labeled demonstrations.
\end{assumption}

All experts are $\epsilon$-experts having an effectively unlimited number of demonstrations. Thus, there will always be unlabeled data showing how to recover from a mistake. A general goal of ILfO is to be able to work with large sets of unlabeled data so we think this simplification is motivated.

\subsection{On-Policy Bound}

If we have a cost-function that measures the relative un-likelihood of state-action pairs $(s,a)$ and we know that the expert is able to recover from mistakes then the amount of time spent off-policy will be relatively small in comparison to the horizon of the task. The assumption that matching the state-action distribution of the expert is enough to guarantee good performance on the task is quite similar to the prevailing assumption underlying model-free adversarial methods to imitation learning \cite{divergence_IL}. Moreover, by assumption \ref{inverse_learn}, looking at state transitions is enough to infer state-action pairs. In this way, we can avoid having to make use of action information in order to evaluate likelihoods.   

\begin{theorem}
\label{Sample Ratio}
Say $J(\pi)$ indicates the expected number of samples labeled as ``non-expert" from the states induced by $\pi$ relative to $\pi^*$. Use $S(\pi) = T - J(\pi)$ to indicate the expected number of samples labeled as ``expert''. Then we have,

\begin{align}
    \frac{S(\pi^*)}{S(\pi)} \le  \cfrac{1}{1-u\epsilon}
\end{align}
\end{theorem}

\begin{proof}
First, note that $J(\pi^*) = T/2$ and by theorem \ref{Ross Theorem} we have,
\begin{subequations}
\begin{align}
    & J(\pi) \le J(\pi^*) + \frac{1}{2} u T \epsilon
\\
    & \Rightarrow S(\pi^*) \le S(\pi) + \frac{1}{2} u T \epsilon
\\
    & \Rightarrow S(\pi^*) - \frac{1}{2} u T \epsilon \le S(\pi)
\\
    & \Rightarrow \cfrac{S(\pi^*)}{S(\pi)} \le \cfrac{S(\pi^*)}{S(\pi^*) - \frac{1}{2} u T \epsilon} = \cfrac{T/2}{T/2 - \frac{1}{2} u T \epsilon} = \cfrac{1}{1 - u \epsilon}
\end{align}
\end{subequations}
\end{proof}

Thus, as long as $\epsilon < 1/u$, we can expect to collect samples that are on policy. Under reasonable assumptions, this is enough to show good properties for BCO. However, if $\epsilon \to 1/u$ the multiplier on the sample complexity will diverge. We introduce a definition to distinguish this case,

\begin{assumption}
The $\epsilon$-expert is said to be recurrent if $\epsilon < 1/u$.
\end{assumption}

\subsection{Iterative BCO}

Our modification of BCO works as follows: start with a small data set, learn an initial policy, and then use that policy to collect data and improve the inverse dynamics model. Next, use the inverse dynamics model to label the demonstrations and improve the policy. Iterate until the policy stops improving. This is an extension of the BCO implementation used by \citet{yang2019imitation}.

\begin{algorithm}
\caption{BCO$^*$}\label{BCO_Algorithm}
\begin{algorithmic}[1]
\State $\text{Initialize the inverse dynamics model } M$
\State $\text{Set } D = \lvert D_0 \rvert $ with $\text{buffer-size} = N$
\While {$\text{policy improvement}$} \do \\
\State Improve $M$ by modelLearning($D$)
\State Use $M$ to compute label($D_\text{demo}$)
\State Improve $\pi$ using $D_\text{demo}$ with labels
\For {time-step $t = 1$ to $N$} \do \\
\State Generate samples $(s_t,a_t,s_{t+1})$ using $\pi$
\State Append to $D$
\EndFor
\EndWhile
\end{algorithmic}
\end{algorithm}

Assume that we have labeled demonstrations $D_0$ that we can use to create an initial policy $\pi_0$ and inverse dynamics model $M_0$ such that $\lvert D_0 \rvert \ll \lvert D_\text{demo} \rvert$. In other words, we have an incomplete demonstration dataset that outputs a policy $\pi_0$ that is essentially unusable. More formally, by Theorem~\ref{Sample Ratio}, if the error rate is very close to $1/u$, then $\pi_0$ will be labeled as spending most of the time off-policy. However, as long at the policy is recurrent, it seems reasonable that we could collect enough rollouts from $\pi_0$ to construct a new labeled data-set $D_1$ to improve the inverse dynamics model.

\subsubsection{Convergence to Behavioral Cloning}

Since we can control how many samples we get using $\pi_0$ we can control how close our training data-set for the inverse dynamics model will be to the optimal policy. At each iteration we get a new policy and a new data set,

\begin{equation}
    \pi_{n+1}, D_{n+1} = \text{BCO}_{\pi_n}(D_\text{demo},D_n) 
\end{equation}

We will spend the rest of this section trying to make these intuitions precise using Theorem~\ref{NTK Generalization} to analyze convergence of BCO. The basic claim is that for a range of $D_0$, we have as $n \to \infty$, 
\begin{equation}
    \text{BCO}^{(n)}(D_\text{demo}) \to \text{BC}(D_\text{labeled \ demo})
\end{equation}

It is worth pointing out the ambiguity regarding whether or not to train the initial policy using BC or BCO. In general, if the policy is easier to learn than the inverse dynamics model, then the initial policy should be made using BC.

\subsection{BCO vs.\ BC}

Although the inverse dynamics model is assumed to be learnable, it might be the case that the inverse dynamics model is simply harder to learn than the policy. Moreover, we need enough initial labeled samples to ensure our initial policy is recurrent. We can formalize this as,

\begin{theorem}
Say we can use $D_0$ to create a recurrent $\epsilon_0$-expert policy. If we create $D_1$ by observing how this policy operates in the environment, we can then learn an $\epsilon_1$-expert policy with probability $1-\delta$ if,
\begin{equation}
    \lvert D_1\rvert \sim O\left( \cfrac{\lVert M^* \rVert_K + \log(1/\delta)}{\epsilon_1^2 (1 - u \epsilon_0)} \right)
\end{equation}
\end{theorem}

In order to make use of the unlabeled data set we will have to create an inverse dynamics model by rolling-out trajectories from the initial policy. We can compare the efficiency of this approach to using BC loosely by comparing error rates,
\begin{subequations}
    \begin{align}
        & \sqrt{\cfrac{2 \lVert M^* \rVert_K}{D_0+D_1 (1-u \epsilon_0)}} = \sqrt{\cfrac{2 \lVert \pi^* \rVert_K}{D_\text{BC}}}\\
        & \Rightarrow D_\text{BC} = (D_0+D_1(1-u\epsilon_0)) \cfrac{\lVert \pi^* \rVert_K}{\lVert M^*\rVert_K}
    \end{align}
\end{subequations}

Thus, the quality of the initial policy and the relative difficulty of learning the policy compared to the inverse dynamics model controls the efficiency of BCO over BC. One limitation of this analysis is that if our initial policy is only barely recurrent then it will take a very large number of samples to train. In fact, \citet{1805.01954} skips the single iteration case and gives a definition of BCO$(\alpha)$ which allows the robot to periodically train during a post-demonstration phase. This suggests that continual learning would be a better strategy. In fact, we'll show that the continual aggregation approach can dampen the impact of a poor initial policy,

\begin{theorem}
    \label{ode}
    Assume the initial policy output by BCO($\alpha$) $\pi_0$ is a recurrent $\epsilon_0$ policy. Following this we then periodically improve the inverse dynamics model as we obtain data $D_1$ from using the policy. Then the sample complexity $\lvert D_0 \rvert + \lvert D_1 \rvert$ to achieve $\epsilon$ error is,
    
    \begin{equation} 
        \lvert D_0 \rvert + \lvert D_1 \rvert \sim O \left( u^2 \lVert M^* \rVert_K \log\left( \cfrac{1/\epsilon - u}{1/\epsilon_0 - u} \right) + \lVert M^* \rVert_K/\epsilon^2 \right) 
    \end{equation}
\end{theorem}

The proof is given in the supplementary material, and involves taking a continuum limit and then solving the resulting ODE.

\section{Experiments}

In this section, we try to verify our analysis by seeing whether or not BCO* is able to improve in performance over successive iterations. For each of the simulated environments, we train experts using Soft-Actor-Critic \cite{brockman2016openai, haarnoja2018soft}. The experts are trained using Stable Baselines with the default hyper-parameter settings \cite{stable-baselines}. Using these trained experts, we generate $30$ demonstrations for the Mountain Car, Ant, and Half-Cheetah environments using Gym and the PyBullet physics simulator \cite{gym,coumans2017pybullet}.

The BC and BCO models are two-layer ReLU MLP networks with layer dimensions of 300 and $200$ respectively. We train the model, hold out $30\%$ of the data for validation, and use checkpoints to select the model with the best validation loss. To evaluate BCO*, we first use a small number of expert demonstrations to generate an initial policy. We then evaluate the policy for 25 episodes and use the resulting data to improve the inverse dynamics model. We choose a larger iteration size for BCO* than Theorem \ref{ode} implies is optimal in order to reduce variance in the training. Note this buffer is far smaller than \citet{1805.01954} used in BCO($\alpha$). Each evaluation of BCO* is averaged over five separate runs of the training process, each with different random seeds. 

For comparisons, we look at the original implementation of BCO that uses random actions to learn the initial inverse dynamics model. Next, we compare against behavioral cloning (BC) trained using 30 demonstrations. We also compare against GAIL and Value-DICE, a newer algorithm that is meant to address the sample complexity issues of GAIL by directly matching the distribution of the robot with that of the expert \cite{kostrikov2019imitation}. Because BC will learn a policy given enough demonstrations, we train Value-Dice using 3 and 15 episodes from Ant and Half Cheetah respectively to compare distribution-matching with the label inference of BCO*. For each comparison, we average over three runs.

\begin{figure}
  \centering
  \includegraphics[width=1.0\textwidth]{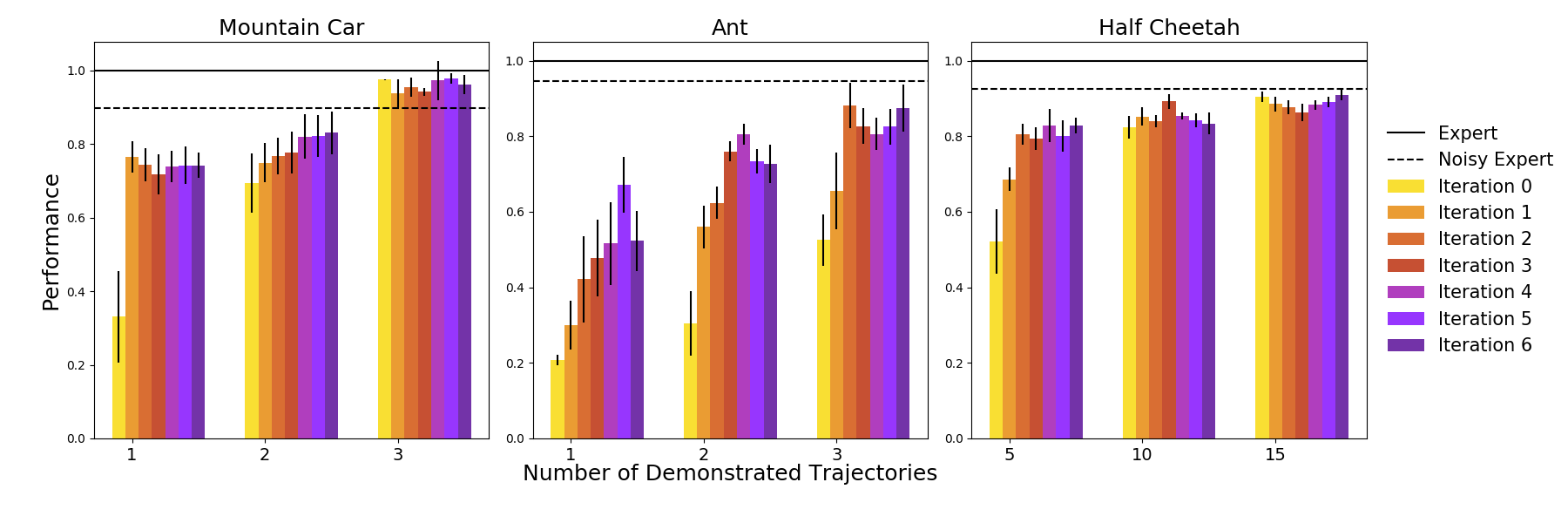}
  \caption{A visualization the evolution of performance during training for different numbers of demonstrations. Performance is measured as the normalized $[0,1]$ reward, where $0$ and $1$ are the mean rewards of random and expert policies, respectively. Error bars represent standard deviation of performance across multiple runs.}
  \label{fig:BCO_Results}
\end{figure}

\subsection{Mountain Car}

For Mountain Car, we use a noisy expert with $\epsilon = 0.01$. However, we note that there is no significant effect on BCO other than a decrease in expert performance. We collect $25$ rollouts in each iteration step and use them to improve the inverse dynamics model. We train all the models at each step for $5000$ epochs and evaluate using $500$ episodes.

\subsection{Ant / Half Cheetah}

For Ant, we use a noisy expert with isotropic Gaussian noise with variance $0.04$. We collect $25$ rollouts from the robot at each iteration step and use them to improve the inverse dynamics model. Besides the initial model, we train all of the models for $50$ epochs to increase efficiency and to avoid the generalization problems with respect to the inverse dynamics model. We evaluate using $50$ episodes at each iteration. 

\subsection{Results}

 Results are shown in Figure~\ref{fig:BCO_Results}. In the case of Mountain Car, the robot converges to near expert-level performance relatively quickly. For Ant, our robot is able to improve over the iterations and to near expert performance by three episodes. In Half Cheetah, the robot is able to perform comparably to the noisy expert. In general, for each of these environments, the performance increases with more iterations and/or demonstrations.

Adding noise to the expert is somewhat essential to our results. If we retrain the robot using noiseless-demonstrations (Fig.~\ref{fig:Perfect_Ant}), the convergence behavior is dampened. Finally, we can compare BCO* against our baselines. We see that BCO* has comparable performance with Value-DICE and BC while GAIL and BCO are too sample inefficient to learn a policy. GAIL, and also GAIfO, can ultimately learn a good policy, however, they are not sample efficient and struggle to learn anything using only 1e6 steps of environment interaction. BCO with a random policy step is also sample inefficient. Typically, we need an order of magnitude more interactions before the random policy will learn a reasonably accurate inverse dynamics model \cite{1805.01954}. 

\begin{figure}
    \centering
  \includegraphics[width=\textwidth]{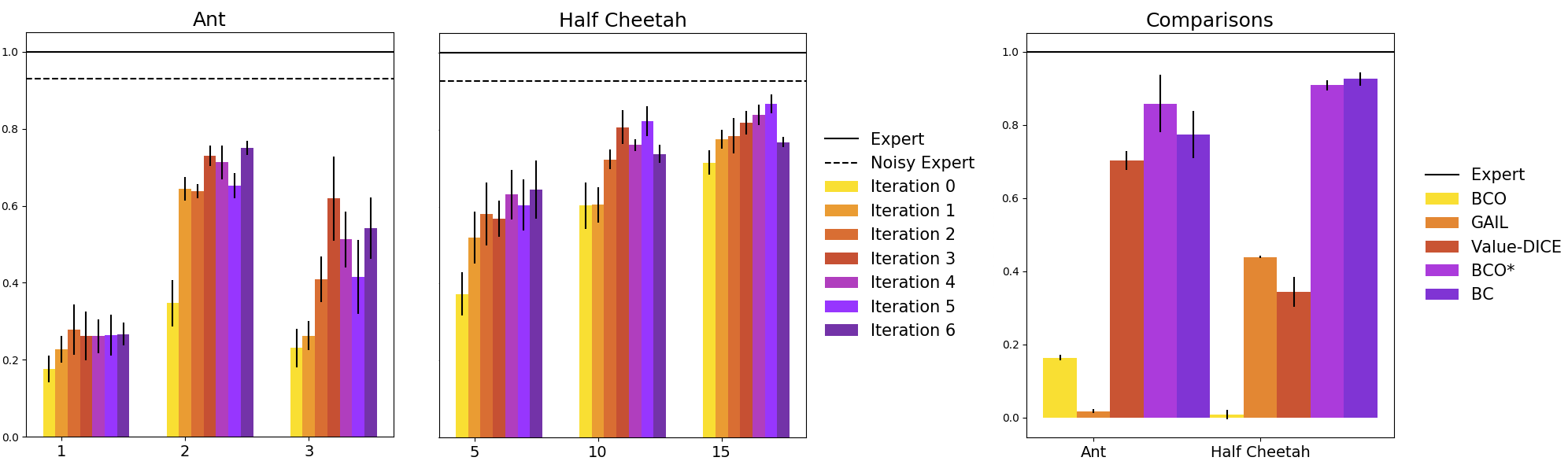}
    \caption{On the left we have runs of BCO where the robot learns from an expert without noise. The x-axis indicates the number of trajectories used for training. On the right we have comparisons against other algorithms using labeled data.}
    \label{fig:Perfect_Ant}
\end{figure}

\section{Discussion}

In this work, we proved that performing BCO* with concurrent training leads to faster improvement under relatively mild assumptions. Our experiments confirm that BCO* is capable of self-improvement in both the idealized setting, where the initial policy is close to the expert and realistic settings, where the initial policy is much farther away. On the other hand, we find that the inverse dynamics model is not always learnable to perfect accuracy, which is in line with results found by \citet{yang2019imitation}. However, this is a limitation of any learning from observation approach. 

Despite these challenges, we think this improvement to BCO offers an alternative paradigm for robot self-improvement. In the real world, labeled data is scarce and expert demonstrations without mistakes are rarer still. We think that these results show the promise of reformulating imitation learning from observation in terms of semi-supervised learning. Up to this point, BCO required an order of magnitude more environment interactions to learn the inverse dynamics model. In addition to being sample inefficient, it isn't always possible to safely execute a random policy in real environments. By introducing a small amount of labeled data we are able to dramatically reduce the sample complexity and to match the performance of behavioral cloning using more than twice as many labels. 

In future work, modifying the inverse dynamics model learning algorithm in order to take advantage of advances in semi-supervised learning could be fruitful~\cite{berthelot2019mixmatch,xie2019unsupervised}. Another avenue for future work is exploring the combination of BCO with RL approaches to ILfO. Methods such as reinforced inverse-dynamics modeling explore this direction with good success~\cite{pavse2019ridm}. 


\clearpage
\acknowledgments{The author would like to thank Chip Schaff for fruitful discussion on experimental details and reviewing and Falcon Dai for technical discussion regarding theoretical results.}


\bibliography{example}  

\newpage

\section{Appendix}

\subsection{Proofs}

\begin{theorem}{(Theorem 3.1)}
Let $\pi$ be such that $\mathbb{E}_{s \sim d_{\pi}}[l(s,\hat \pi)] = \epsilon$ and $Q_{T-t+1}^{\pi^*}(s,a) - Q_{T-t+1}^{\pi^*}(s,\pi^*) \le \frac{1}{2} u$ for all times and actions where $d_{\pi}^t(s) > 0$, then $J(\pi) \le J(\pi^*) + \frac{1}{2} u T \epsilon$
\end{theorem}

\begin{proof}
Given the policy $\pi$, we consider the policy $\pi_{1:t}$ that executes $\pi$ in the first $t$-steps and then executes the expert following that. Then,

\begin{subequations} \label{eqn:cost-empirical}
    \begin{align}
    J(\hat \pi) &= J(\pi^*) + \sum_t J(\pi_{1:T-t})-J(\pi_{1:T-t-1}) \quad \text{(Telescoping)}\\ 
    &= J(\pi^*) + \sum_t \mathbb{E}_{d^t_{\hat \pi}}[Q_{T-t+1}(s,\hat \pi) - Q_{T-t+1}(s,\pi^*)]\\
    &\le J(\pi^*) + \frac{1}{2} u \sum_{t=1}^T \mathbb{E}_{d^t_{\hat \pi}}[l(s,\pi]]\\
    &\le J(\pi^*) + \frac{1}{2} u T \epsilon
    \end{align}
\end{subequations}

\end{proof}

\begin{theorem}
\label{BCO Recurrence}

Let $\mathbb{D}$ indicate the distribution of trajectories generated by an $\epsilon$- expert policy. Let $\epsilon^*$ by the least upper bound such that the expert is recurrent. Also let $D = \lbrace{ \zeta_1, \zeta_2, \ldots \rbrace}$ be a sequence of the expert demonstrations where each $\zeta_i$ is a sequence of unlabeled demonstrations $\lbrace (s_0), \ldots, (s_N) \rbrace$. Similarly, let $D_0$ indicate a labeled set of demonstrations such that $\lvert D_0 \rvert << \lvert D \rvert$. Finally, let $K$ indicate the complexity for learning a policy/inverse dynamics model. If we obtain a labeled set of demonstrations such that,

\begin{align} \lvert D_0 \rvert \sim O \left(\cfrac{(K + \log(1/\delta))}{\epsilon^{* 2}} \right)
\end{align}

\noindent then with probability $1-\delta$ we can learn a recurrent policy $\hat \pi$.

\end{theorem}

\begin{proof}

We start by sampling a few demonstrations $D_0$ from the expert state-action distribution $\mathbb{D}$. Since we assume the unlabeled demonstration set is large the policy itself will have an error equal to it's inverse dynamics model error. If the inverse dynamics model is easy to learn then we can train the inverse dynamics model and then label the unlabeled demonstrations. Otherwise, we can just directly learn a policy. We can check that $\lvert D_0 \rvert$ is large enough directly by using assumption 4.4. 

\begin{equation}
\epsilon \le \sqrt{K/n} + \sqrt{\log(1/\delta)/n}
\end{equation}

By AM-GM we have,
\begin{equation}
    \epsilon \le \sqrt{\cfrac{2 (K + \log(1/\delta))}{n}}
\end{equation}
In order for a policy to be recurrent we need $\epsilon < \epsilon^*$. This condition is satisfied when,
\begin{align}
\epsilon^* = \sqrt{\cfrac{2 (K + \log(1/\delta))}{n}}\\ 
\Rightarrow 1/\epsilon^{*2} = \cfrac{2 (K + \log(1/\delta))}{n}\\ 
\Rightarrow n \ge \cfrac{2 (K + \log(1/\delta))}{\epsilon^{*2}}
\end{align}
Then as long as $\lvert D_0 \rvert$ is $O \left(\cfrac{(K + \log(1/\delta))}{\epsilon^{* 2}} \right)$ then with probability $1-\delta$ the resulting policy will be recurrent.

\end{proof}

\begin{theorem}{(Theorem 4.9)}
Say we can use $D_0$ to create a recurrent $\epsilon_0$-expert policy. If we create $D_1$ by observing how this policy operates in the environment we can then learn an $\epsilon_1$-expert policy with probability $1-\delta$ if,
\begin{equation}
    \lvert D_1 \rvert \sim O\left( \cfrac{\lVert M^* \rVert_K + \log(1/\delta)}{\epsilon_1^2 (1 - u \epsilon_0)} \right)
\end{equation}
\end{theorem}

\begin{proof}

By theorem 4.7 the number of useful samples in $D_1$ is $D_1 (1 - u \epsilon)$. From here, all we need to do is solve,

\begin{subequations}
    \begin{align}
        & \sqrt{\cfrac{2 (\lVert M^*\rVert_K + \log(1/\delta))}{\lvert D_0 \rvert + \lvert D_1 \rvert (1-u \epsilon_0)}} = \epsilon_1 \\
        & \Rightarrow \cfrac{2 (\lVert M^* \rVert_K + \log(1/\delta))}{\epsilon_1^2} = \lvert D_0 \rvert + \lvert D_1 \rvert ( 1 - u\epsilon_0)\\
        & \Rightarrow \lvert D_1 \rvert \sim O\left( \cfrac{\lVert M^*\rVert_K + \log(1/\delta)}{\epsilon_1^2 (1 - u \epsilon_0)}\right)
    \end{align}
\end{subequations}

\end{proof}

\begin{theorem}{(Theorem 4.10)}
    Assume the initial policy output by BCO($\alpha$) $\pi_0$ is recurrent. Then the sample complexity to achieve $\epsilon$ error is,
    
    $$N \sim O\left(\log \left( \cfrac{2 - (1 -\epsilon)^{-u}}{2 - (1 -\epsilon_0)^{-u}} \right) + \beta^{-1} L^2 \lVert M^*\rVert / \epsilon^2 \right)$$
\end{theorem}

\begin{proof}
The proof boils down to solving a non-linear ODE and showing that the only possibility for divergence is logarithmic which is apparent after changing variables and decomposing the integral into partial-fractions. We can solve the ODE with integration. However, we don't care about multiplying constants so we'll omit them whenever convenient.

Say $N_k$ is the number of accumulated on-policy samples we have after $k$-sessions of interacting with the environment. Then we expect,

\begin{align}
    \epsilon_{k+1} = \sqrt{\frac{\lVert M^*\rVert_K}{N_k}}, \quad N_{k} = N_{k-1} + \alpha (1-u \epsilon_k)
\end{align}

\noindent This update strategy gives us $BCO(\alpha)$. We can analyze this via the continuous approximation,

\begin{align}
    \begin{aligned}
        \delta N(\alpha k) = \alpha (1-u \epsilon_k) \quad \\
        \Rightarrow \frac{dN}{dt} = (1- u \epsilon_k), \quad \alpha \to dt
    \end{aligned}
\end{align}

\noindent Now we can write,

\begin{align}
    \begin{aligned}
         \frac{dN}{dt} = (1 - u \epsilon) = \left(1 - \sqrt{\frac{\lVert M^*\rVert_K}{N(t)}} \right), \quad N(0) = N_0
    \end{aligned}
\end{align}

This is a differential align that relates the number of effective samples $N$ we have after interacting with the environment for $t$ time-steps. As $N \to \infty$ we see that $\frac{dN}{dt} \to 1$ which indicates that BCO recovers the true policy in the limit and spends more time on policy than the demonstrations do. We can solve the ODE by integrating,

$$\int \cfrac{dN}{1 - u \sqrt{\lVert M^*\rVert_K/N}} = 2 u^2 \lVert M^*\rVert_K \log(\sqrt{N} - u \sqrt{\lVert M^*\rVert_K}) + 2 u \sqrt{\lVert M^*\rVert_K N} + N = t$$

We can substitute $N = \lVert M^*\rVert /\epsilon^2$ and the initial condition to arrive at,

$$ j\Delta t  = 2 u^2 \lVert M^* \rVert_K \log\left( \cfrac{1/\epsilon - u}{1/\epsilon_0 - u} \right) + 2 u \lVert M^*\rVert_K (1/\epsilon - 1/\epsilon_0) + \lVert M^*\rVert_K (1/\epsilon^2-1/\epsilon_0)$$

While we still have a divergence when $\epsilon_0 \to 1/u$ it's now been suppressed to logarithmic. Asymptotically, if the number of samples $D_1 = t$ then we have,

$$\lvert D_1 \rvert \sim O \left( u^2 \lVert M^*\rVert_K \log\left( \cfrac{1/\epsilon - u}{1/\epsilon_0 - u} \right) + \lVert M^*\rVert_K (1/\epsilon^2-1/\epsilon_0^2) \right)$$

\end{proof}

\end{document}